\documentclass{article}
\usepackage[utf8]{inputenc}

\usepackage[hidelinks]{hyperref}
\usepackage{amsmath,amsfonts,amsthm,cleveref}
\usepackage[a4paper,top=2.9cm,bottom=2.5cm,left=2.5cm,right=2.5cm,marginparwidth=2cm]{geometry}
\usepackage{subcaption}

\newcommand{\R}{\mathbb{R}}

\newtheorem{theorem}{Theorem}[section]
\newtheorem{corollary}[theorem]{Corollary}
\newtheorem{lemma}[theorem]{Lemma}

\usepackage{todonotes}

\title{Bioinspired random projections for robust, sparse classification}
\author{Nina Dekoninck Bruhin\thanks{Department of Mathematics, ETH Zurich, Switzerland} \and Bryn Davies\thanks{Department of Mathematics, Imperial College London, UK}}
\date{}

\begin{document}
\maketitle

\begin{abstract}
    Inspired by the use of random projections in biological sensing systems, we present a new algorithm for processing data in classification problems. This is based on observations of the human brain and the fruit fly's olfactory system and involves randomly projecting data into a space of greatly increased dimension before applying a cap operation to truncate the smaller entries. This leads to a simple algorithm that is very computationally efficient and can be used to either give a sparse representation with minimal loss in classification accuracy or give improved robustness, in the sense that classification accuracy is improved when noise is added to the data. This is demonstrated with numerical experiments, which supplement theoretical results demonstrating that the resulting signal transform is continuous and invertible, in an appropriate sense.
\end{abstract}

\vspace{0.5cm}


\vspace{0.2cm}


\vspace{0.5cm}

\section{Introduction}

Through millennia of evolution, biology has developed sensing architectures that are remarkably versatile in the face of diverse classification problems. As a result, biological systems represent the gold standard in many fields and are often the benchmark against which other solutions are compared. As well as giving impressive success rates for challenging classification problems, biological systems often give very sparse representations of data and perform well in the presence of significant noise. It is these abilities that the simple algorithm developed in this work is designed to replicate.

The remarkable performance of biological systems is not unique to signal processing and data classification problems, and efforts to replicate their function have given rise to the science of \emph{biomimicry} \cite{benyus1997biomimicry}. The fundamental principle is to seek innovative solutions to challenging problems by taking inspiration from nature's remarkable achievements. Famous success stories include optimising the aerodynamic design of high-speed trains, designing passive climate control systems for buildings and the invention of hook and loop fasteners such as Velcro \cite{benyus1997biomimicry, biomimicry}. In signal processing, there have also been numerous attempts to replicate biological systems. For example, researchers have constructed sophisticated sonar and echolocation devices \cite{hodges2011underwater}, replicated the communication and perception of electric fish \cite{baldassari2021multi, gottwald2019bio}, designed metamaterial spectrometers that replicate the cochlea \cite{ammari2020biomimetic, ammari2020mimicking, rupin2019mimicking}, found algorithms that retain biological sensing invariants \cite{bruna2013invariant, mallat2012group, deepspectrum} and, perhaps most famously of all, designed network architectures reminiscent of neural processing \cite{nielsen2015neural, lecun2015deep}.

In this work, we will focus our attention on the use of random projections in biological sensing systems. It is often the case that connections between neurons are formed at random and they differ between different organisms. The fruit fly's olfactory system is an example that is particularly well understood \cite{neuralAlgorithm, stevens, stevensStatistical} and there is also some evidence that the human brain behaves similarly \cite{stevens}. An approach of this type has the advantage that you do not need to undertake expensive training periods in order to learn the connections. Instead, you randomly form many connections and try to select the best ones. The method we develop in this work will be closely inspired by the fruit fly's olfactory system \cite{neuralAlgorithm, stevens, stevensStatistical}. It will have two steps: first, a feature vector will be projected into a high-dimensional space by multiplying by a random matrix. This matrix will be rectangular, such that the dimension of the image space is much larger than that of the initial feature vector. Following this, we will apply a cap operation, that sets all but the few largest entries in the vector to zero. This means the final representation can be chosen to be suitably sparse. 

Random projections have been used in a variety of computational applications, often with the aim of reducing computational time. This is particularly the case when they are used to replace expensive training steps. For example, they have been used in neural networks to randomly select weights \cite{robustlearning} or features \cite{KitchenSinks}. In both cases, there is little loss of performance with a significant reduction in computation time. Similarly, random projections have been used to reduce dimensionality, either as the first step in a classification algorithm \cite{visualCategorization} or to approximate kernel functions \cite{RandomKernelMachines}. In many cases, the structure of the random projections is specifically chosen for the task in hand (\emph{e.g.} to detect corners in images, as in \cite{visualCategorization}). Conversely, the method developed in this work will be, theoretically, independent of the setting and will treat the problem of classifying a given feature vector of arbitrary dimension.

As well as reducing computation time, there are also deep connections between the use of random projections and classification robustness. For example, \cite{robustlearning} shows that more robust targets (in a suitable sense) can be more effectively compressed by being randomly projected to a lower dimensional space. As a result, concepts that are sufficiently robust can be successfully randomly projected to a lower dimension, where they can be classified. In this work, we explore the extent to which our bioinspired algorithm facilitates robust classification. In particular, we show that, for the problem of musical genre classification, the use of random projections gives an improvement in classification accuracy when random noise is added to the signal. 

This article will begin by exploring the biological motivations for our algorithm in more detail. Once properly motivated, we will explore the mathematical properties of our bioinspired random projections and prove results that characterise the extent to which the resulting transformation is continuous and invertible. Finally, we will perform numerical experiments on the toy problem of musical genre classification. These experiments demonstrate that our bioinspired transformation, when used in partnership with a support vector machine, not only retains the classification accuracy but potentially improves on it, particularly when noise is added to the data.

\section{Biological motivations} \label{sec:biology}

We will develop an approach that is based on how random projections are used in biological systems. In this section, we summarise the main observations that will influence the design of our algorithm.

\subsection{The fruit fly's olfactory system}

Early processing of odors in the fruit fly's olfactory system consists of roughly three steps. In the first step, olfactory receptor neurons (ORNs) located in the fly's antennae detect an odor and send a signal to projector neurons in the antennal lobe. In the second step, the projector neurons transmit the firing rates to parts of the fly's brain known as the mushroom body and the lateral horn. We will focus on the mushroom body, since this part of the fly's brain is known to be important for learning new smells and creating memories associated with them \cite{olfactoryinfo}. Here, signals are transmitted randomly to a large number of Kenyon cells. Finally, in the third step, anterior paired lateral neurons suppress a large number of the Kenyon cells, so that only those with the highest firing rates are uninhibited. The important parts of the fly's olfactory system are shown in Figure~\ref{fig:fly_diagram}(a) (which is reproduced from \cite{perisse2013shocking}), along with a sketch of the connections in Figure~\ref{fig:fly_diagram}(b).

\begin{figure}
    \centering
    \begin{tikzpicture}
        \node[inner sep=0pt] (fly) at (0,0)
    {\includegraphics[width=.5\textwidth]{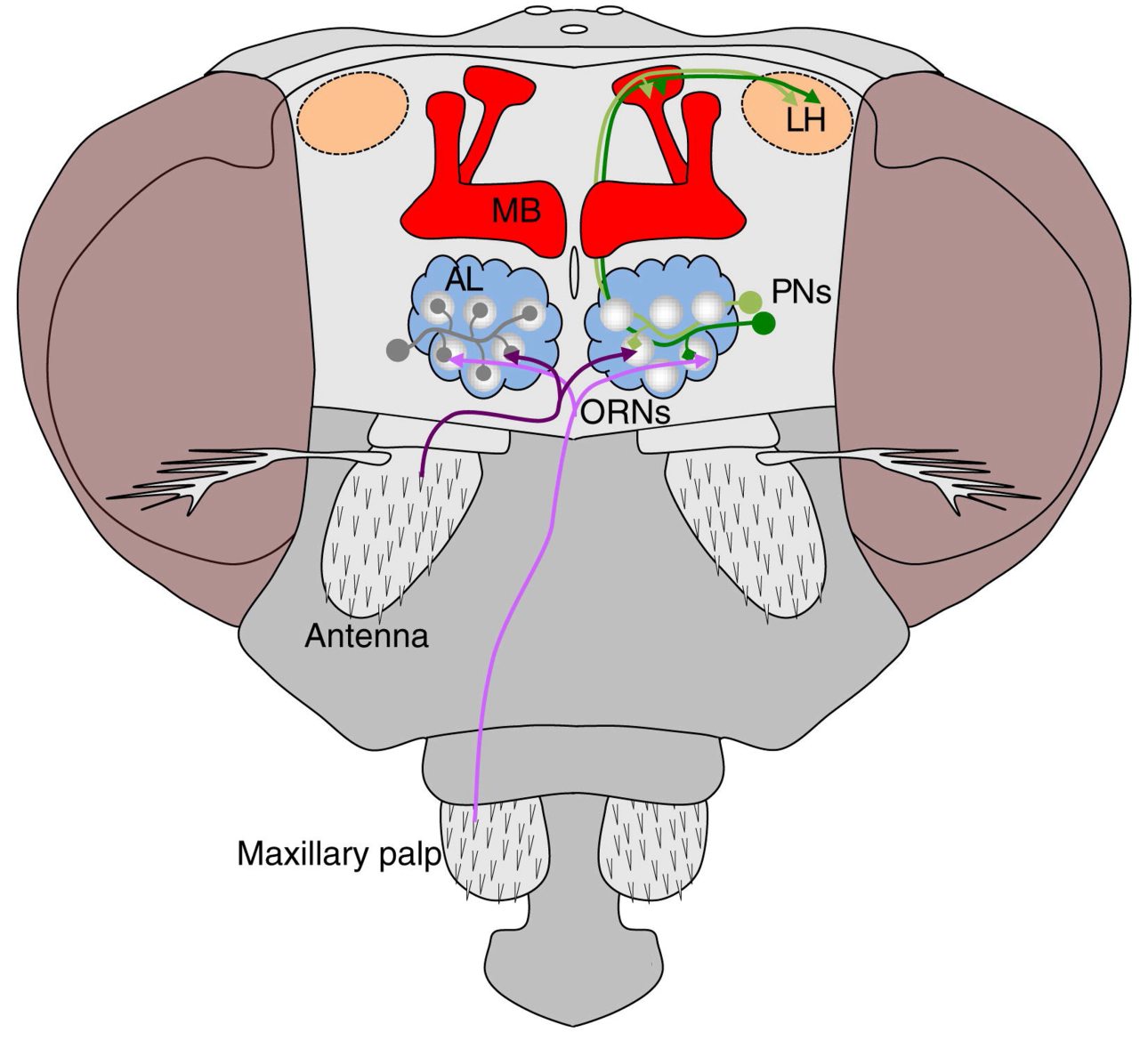}};
        \node[inner sep=0pt] (network) at (8,0)
    {\includegraphics[width=.35\textwidth]{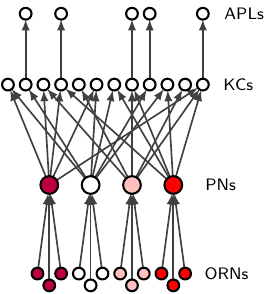}};
    \node[inner sep=0pt] at (-3.5,-2) {(a)};
    \node[inner sep=0pt] at (5,-2) {(b)};
    \end{tikzpicture}
    \caption{The first steps of the fruit fly's olfactory system, shown (a) in a sketch of the fly's head and (b) as a simplified schematic. Odors are detected by olfactory receptor neurons (ORNs) in the antennae and the maxillary palps. ORNs of the same type fire to the same projector neurons (PNs) in the antennal lobe (AL). Projector neurons then fire to the mushroom body (MB) and the lateral horn (LH). In the MB,  signals are transmitted in random combinations to a large number of Kenyon cells (KCs). Finally, anterior paired lateral (APL) neurons inhibit the output of 95\% of the KCs, leaving only those with the largest firing rates. The image in (a) is reproduced from Perisse~et~al.~\cite{perisse2013shocking} with the permission of Elsevier.}
    \label{fig:fly_diagram}
\end{figure}

Smells are detected when volatile molecules attach themselves to protein receptors of olfactory receptor neurons (ORNs) in the fly's antennae \cite{olfactoryinfo}. Each ORN possesses a particular set of olfactory receptors, that determines which smells it will be activated by. There are roughly 50 different types of ORNs in total. The activation of ORNs is dependent on the concentration of the odor. However, flies' behaviour has been shown to be invariant over a wide range of odor intensities \cite{concentrationInvariant, stevensStatistical}, suggesting that subsequent steps in the fly's olfactory system are able to partially remove this dependency on concentration. In particular, \cite{concentrationInvariant} shows that increasing the concentration of an odor triggers more ORNs, that have different sensibilities to the odor. The firing of multiple ORNs in turn activate inhibitory neurons, so that the final firing rates of the projector neurons stay relatively stable over a large range of odor concentrations. The consequence of this is that we elect to disregard the details of this step in our mathematical analogue algorithm, and begin with a feature vector obtained by calculating an appropriate basis decomposition. In our experiments, we will use the scattering transform \cite{bruna2013invariant}.

After being activated, ORNs fire to structures called glomeruli in the antennal lobe. Each glomerulus receives the input from all ORNs of a particular type; there are therefore about 50 glomeruli. In the glomerulus, ORNs make synaptic contact with a projector neuron.  At this stage, the odor information thus can be represented as a 50-dimensional vector, where each coefficient corresponds to the firing rate of a single type of ORN. That signal is then projected to 2000 Kenyon cells in the mushroom body, resulting in a 40-fold increase in signal dimension. Each Kenyon cell receives the firing rates of approximately 6 projector neurons and sums them up \cite{neuralAlgorithm}. Crucially, the projection to Kenyon cells is random, in the sense that the latter do not receive a signal from fixed projector neurons depending on the type of smell detected. From one fly to the other, a similar smell triggers different Kenyon cells, even if the same types of ORNs have been activated. 

Evidence suggests that the set of Kenyon cells activated after exposure to an odor forms the odor ``tag" that allows the fly to recognize it: \cite{campbell} shows that a large overlap in the firing rates of a group of Kenyon cells is a good predictor of whether a fly will judge two smells to be similar. Moreover, the results of \cite{campbell} show that the entire population of Kenyon cells is not necessary to discriminate smells, but rather that a subset of 25 cells gives sufficient information to predict the fly’s response. This is a consequence of the fact that glomeruli fire randomly to Kenyon cells and the results are summed. Thus, the entire information they project can be found in a relatively small subset of the Kenyon cells. One can then wonder why the information is spread over 2000 Kenyon cells, when many fewer seem to provide enough information. Stevens \cite{stevens} argues that the reason for such a large, redundant representation of smells in the mushroom body is to provide multiple representations, so that the fly can later utilise the one containing the crucial information.

Finally, the last step in the fly’s olfactory system consists of an inhibitory process. Anterior paired lateral (APL) neurons deactivate about 95\% of Kenyon cells, leaving only those with the highest firing rates \cite{neuralAlgorithm, stevens}. As a result, the final vector representation of the smell information is relatively large (a vector with 2000 entries) and very sparse. This is easy to reproduce with a cap operation, that acts in the same way by setting all but the largest entries to zero.

Altogether, those steps amount to a random projection of the initial 50-dimensional vector into a 2000-dimensional space, followed by a non-linear operation that only keeps a fixed number of the highest coefficients. The initial 50-dimensional signal vector contains the firing rates of each type of ORN. The projection of the glomeruli to the Kenyon cells can be described by multiplication by a random matrix with size $2000 \times 50$ and entries drawn from $\{0,1\}$. Each row corresponds to a particular Kenyon cell; for each glomerulus that fires a signal to that cell, we write a 1 in the corresponding column of the row vector. All other entries are set to 0. As only few of the projector neurons fire to the next step, the random matrix should be sparse. We modify this formulation slightly in our algorithm, to take advantage of other beneficial properties (\emph{e.g.} when the entries have a symmetric distribution).

\subsection{The human brain}

Interestingly, the application of a random projection followed by an inhibitory process is not unique to the fly’s olfactory system. In fact, such a process plays a role in three parts of the brain: the cerebellum, the hippocampus and the olfactory system.  Stevens \cite{stevens} explains that the way those three structures process information follows a similar three-step architecture to the fly's olfactory system. In the first stage, the information arriving from other brain areas is assembled into a neural code. In the second stage, that code is passed on to a greatly enlarged number of neurons. Finally, in the third stage, this code is broken down to be interpreted in further information processing steps. 

The analogy is particularly clear in the functioning of the cerebellum. This is a large part of the brain that has several important roles, including receiving information from sensory systems and integrating these inputs to fine-tune motor activity \cite{fine2002history}. The initial information code is provided by precerebellar neurons, before being passed on to a much larger number of granule cells. The output of granule cells is then passed through Golgi inhibitory neurons and fed back to the granule cells. Finally, the information code is relayed to Purkinje cells, see \cite{stevens} for more details.

A related model of the brain, which uses random projections and cap operations, was developed by Papadimitriou and Vempala \cite{assemblies} and is worth mentioning at this point. It has been shown that we learn complicated concepts or create memories thanks to the action of assemblies of neurons, where an assembly is defined as a set of densely interconnected neurons that fire almost simultaneously when the associated concept or idea is thought about \cite{assemblies}. Aiming to describe the brain’s function with such assemblies, Papadimitriou and Vempala \cite{assemblies} define notions of projection and capping, which explain experimental data. They describe projection as the repetitive firing of an assembly to a different part of the brain, in order to eventually form a new assembly. Once that link is formed, the firing of the parent assembly will trigger the firing of the child assembly. Their version of a cap operator, which prevents all neurons in a brain area from firing except the $k$ with highest synaptic input, is similar to the fly’s inhibitory process. This model takes the form of a bipartite directed $G_{n,p}$ graph, which Papadimitriou and Vempala \cite{assemblies} assume fires at discrete time steps. This gives a sequence of successive graphs, which evolve over time as the system learns and connections are formed and then reinforced. It can be show that, under suitable assumptions, this sequence will converge in the sense that no new neurons will fire after a sufficient number of time steps. 

This model of projection is much more complex than the one we will consider, as the initial stimulus fires repetitively instead of only once, and the triggered set of neurons in turn fires, further complicating the model. However, both models share some significant similarities. Firstly, they are both random, in the sense that edges in the graph presented by Papadimitriou and Vempala exist randomly and independently with a given probability. Additionally, their use of a cap operation, which prevents all neurons in a brain area from firing except the $k$ with highest synaptic input, is very similar to the fly’s inhibitory process.

\section{Mathematical properties} \label{sec:maths}

Motivated by the above discussion, we will consider the transformation $A:\R^m\to\R^n$ given by
\begin{equation} \label{eq:transformation}
    A(s)=c_k(Ms),
\end{equation}
where $M\in\R^{n\times m}$ is a random matrix and $c_k:\R^n\to\R^n$ is a cap operation that keeps the $k$ largest (in magnitude) entries of a vector and sets all the others to zero. Implicitly, we need $k\leq n$. On top of this, we will focus on the case $n\gg m$, to replicate the random projection from few projector neurons to many Kenyon cells in the fly's olfactory system. We will choose $M\in\R^{n\times m}$ to be such that is a random matrix whose entries $m_{ij}$ are independent and identically distributed and given by the difference between two independent Bernoulli random variables:
\begin{equation}
    m_{ij} = X_{ij} - Y_{ij}, \quad i=1,\dots,n,\, j=1,\dots,m,
\end{equation}
where $X_{ij}$ and $Y_{ij}$ are independent Bernoulli random variables with parameter $p \in (0,1)$:
\begin{equation}
  X_{ij},Y_{ij} = \begin{cases}
       1 & \text{with probability } p\\
                               
        0 & \text{with probability  } 1-p.\\
        \end{cases}
\end{equation}
This choice of random matrix is motivated by the way that, in the fly's olfactory system, each Kenyon cell receives the firing rates of multiple projector neurons and sums them up. However, we have added the extra feature that each entry of $M$ should have mean zero, which will yield several useful mathematical properties using the existing literature on properties of random matrices. Similar matrices with symmetric distributions were considered \emph{e.g.} by \cite{databasefriendly, visualCategorization, robustlearning, faceRecognition}. We will choose $p\in(0,1)$ to be small, such that $M$ is likely to be sparse. This will improve the speed of calculations, especially when $M$ is very large.

Intuitively, we would like our transformation $A$ to satisfy certain characteristics. Firstly, we want it to be continuous in the sense that similar signals should still be close after transformation. On the other hand, we want signals that are different enough to be further apart after being transformed by $A$. While the latter is a bit more complex to guarantee, we will be able to show that, with high probability, our transformation preserves similarities between vectors.

\subsection{Properties of the random projections}

We begin by exploring the properties of multiplication by the random matrix $M$. The following lemma describes the distribution of the entries of the matrix $M$:

\begin{lemma} \label{lem:bernoulli}
If $m_{ij} \sim X-Y$, where $X$,$Y$ are Bernoulli random variables with parameter $p \in (0,1)$, then it holds that $\mathbb{E}(m_{ij}) = 0$, $\mathbb{P}(m_{ij}=0) = 2p^2-2p+1$ and $\mathrm{Var}(m_{ij}) = 2p(1-p)$. 
\end{lemma}
\begin{proof}
Let $Z = X-Y$ where $X$,$Y$ are Bernoulli random variables with parameter $p \in (0,1)$. Then, the expectation follows by a simple calculation: $\mathbb{E}(Z) = \mathbb{E}(X-Y) = \mathbb{E}(X) - \mathbb{E}(Y) = p-p  =0$. Similarly, we can calculate that $\mathbb{P}(Z=0) = \mathbb{P}(X=0,Y=0)+\mathbb{P}(X=1,Y=1)=(1-p)^2+p^2$. For the variance, we note that since $Z \in \{-1,0,1\}$, we have $Z^2 \in \{0,1\}$. 
Thus $\mathbb{E}(Z^2) = \mathbb{P}(Z^2 = 1) = 1-
\mathbb{P}(Z = 0) = 2p(1-p)$. Since $\mathbb{E}(Z)=0$, we have $\mathrm{Var}(Z) = \mathbb{E}(Z^2) = 2p(1-p)$.
\end{proof}

We first present a theorem that bounds the operator norm of the random matrix $M$ with high probability: 
\begin{theorem} [Upper tail estimates]
Given the matrix $M\in\R^{n\times m}$, whose entries are each the difference of independent and identically distributed Bernoulli random variables, there exist real-valued constants $C$ and $c>0$ such that 
$$\mathbf{P}( \| M \|_{op} > D \sqrt{n}) \leq C \exp{(-cDn)},$$
for all $D \geq C$. In particular, we have $ \| M \|_{op} = O( \sqrt{n})$ with probability that is exponentially close to 1. 
\end{theorem}
\begin{proof}
This theorem and its proof can be found in \cite[Theorem 2.1.3]{taobook}, for the more general case where $M$ is any matrix with entries that are independent and identically distributed, have zero mean and are uniformly bounded by $1$. From \Cref{lem:bernoulli} and by inspection, we can see that this holds for our specific choice of $M$.
\end{proof}

This theorem gives us a bound on
\begin{equation}
\|Mx\| \leq \|M\|_{op} \|x\|,
\end{equation}
which guarantees that, with high probability, a vector's norm won't blow up due to multiplication by the random matrix $M$, even as the dimension of $M$ becomes very large. Throughout this article, we will use $\|\cdot\|$ to denote the Euclidean norm (\emph{i.e.} $\|\cdot\|_2$). In particular, it holds that 
\begin{equation}
\|Mx - My\| \leq \|M\|_{op} \|x-y\|    
\end{equation}
for any $x,y\in\mathbb{R}^m$.

We will however be able to give further bounds on $\|Mx- My\|$ by modifying a famous result by Johnson and Lindenstrauss, which states that a set of points in $\mathbb{R}^d$ can be mapped to $\mathbb{R}^k$ while approximately preserving distances between pairs of points, as long as $k$ is large enough. Most of the literature focuses on the case where $k < d$, as this allows for data compression; however, this lemma is still informative to us, as it shows that we can map data points to a different dimension while more or less retaining their pairwise distances. Note that the standard formulation of this lemma only states that such a mapping exists, and does not specify what it might look like. The Johnson-Lindenstrauss lemma can be found in \cite[Lemma 1.1]{databasefriendly} and says the following:

\begin{theorem} [Johnson-Lindenstrauss Lemma] \label{johnsonLindenstrauss}
Given $\epsilon > 0$ and an integer $n$, let $k$ be a positive integer such that $k \geq k_0 = O(\epsilon^{-2} \log n)$. For every set $P$ of $n$ points in $\mathbb{R}^d$ there exists $f:\mathbb{R}^d \longrightarrow \mathbb{R}^k$ such that for all $u,v \in P$
$$(1- \epsilon) \|u-v\|^2 \leq \|f(u)-f(v)\|^2 \leq (1 + \epsilon) \|u-v\|^2.$$
\end{theorem}

The next theorem shows that multiplication by a random matrix $R$, whose entries follow a distribution that is symmetric around zero and has unit variance, preserves the norm of vectors up to a scaling constant, with high probability. This constant depends on $\sqrt{n}$, where $n$ is the dimension of the space into which $M\in\R^{n\times m}$ projects. The theorem and its proof come from \cite[Theorem 1]{robustlearning}. 

\begin{theorem} \label{theorem 3.7}
Let $R\in\R^{n\times m}$ be a random matrix, with each entry $r_{ij}$ chosen independently from a distribution that is symmetric about the origin with $\mathbb{E}(r_{ij}^2) = 1$. 

\begin{enumerate}
    \item Suppose $B = \mathbb{E}(r_{ij}^4) < \infty$. Then, for any $\epsilon > 0$,
    $$\mathbb{P}\left(\|\frac{1}{\sqrt{n}}Ru\|^2 \leq (1- \epsilon)\|u\|^2\right) \leq \exp\left(- \frac{(\epsilon^2 - \epsilon^3)n}{2(B+1)}\right)\qquad\text{ for all }u\in\mathbb{R}^m.$$
    \item Suppose $\exists L > 0$ such that for any integer $k > 0$, $\mathbb{E}(r_{ij}^{2k}) \leq \frac{(2k)!}{2^k k!} L^{2k}$. Then, for any $\epsilon > 0$,
    $$\mathbb{P}\left(\|\frac{1}{\sqrt{n}}Ru\|^2 \geq (1+\epsilon)L^2\|u\|^2\right) \leq \exp\left(-(\epsilon^2 - \epsilon^3)\frac{n}{4}\right)\qquad\text{ for all }u\in\mathbb{R}^m. $$
\end{enumerate}
\end{theorem}

As shown in Lemma~\ref{lem:bernoulli}, our matrix $M$ satisfies $\mathbb{E}(m_{ij}^2) = 2p(1-p)\neq1$, therefore we need the following corollary to extend Theorem~\ref{theorem 3.7} to settings with matrix entries drawn from a distribution with rescaled variance. 

\begin{corollary} \label{corollary_3_8}
Let $M\in\mathbb{R}^{n \times m}$ be a random matrix whose entries $m_{ij}$ are sampled independently and randomly from a distribution that is symmetric around the origin with $\mathbb{E}(m_{ij}^2)=\sigma^2>0$.
\begin{enumerate}
    \item Suppose $B = \mathbb{E}(m_{ij}^4) < \infty$. Then, for any $\epsilon > 0$,
    $$\mathbb{P}\left(\| \frac{1}{\sqrt{n}} Mu\|^2 \leq \sigma^2(1- \epsilon)\|u\|^2\right) \leq \exp\left(- \frac{(\epsilon^2 - \epsilon^3)n}{2(\frac{1}{\sigma^4} B+1)}\right), \quad\text{for all } u \in \mathbb{R}^m.$$
    \item Suppose $\exists L > 0$ such that for any integer $k > 0$, $\mathbb{E}(m_{ij}^{2k}) \leq \sigma^{2k} \frac{(2k)!}{2^k k!} L^{2k}$. Then, for any $\epsilon > 0$,
    $$\mathbb{P}\left(\|\frac{1}{\sqrt{n}}M u\|^2 \geq \sigma^2 (1+\epsilon)L^2\|u\|^2\right) \leq \exp\left(-(\epsilon^2 - \epsilon^3)\frac{n}{4}\right), \quad\text{for all } u \in \mathbb{R}^m. $$
\end{enumerate}
\end{corollary}

\begin{proof}
1. Let $R$ be the $n \times m$ matrix defined as $R := \frac{1}{\sigma} M$. Then clearly the entries $r_{ij}= \frac{1}{m_{ij}}$ of $R$ are sampled from a distribution symmetric around $0$. Moreover, $\mathbb{E}(r_{ij}^2)=\mathbb{E}(\frac{1}{\sigma^2} m_{ij}^2)=\frac{\sigma^2}{\sigma^2}=1$. Finally, $\mathbb{E}(r_{ij}^4)=\mathbb{E}(\frac{1}{\sigma^4}m_{ij}^4) = \frac{1}{\sigma^4}B < \infty$. Hence $R$ satisfies the conditions of Theorem~\ref{theorem 3.7}, and we have
\begin{equation}
\mathbb{P}(\| \frac{1}{\sqrt{n}} R u \|^2 \leq (1- \epsilon)\|u\|^2) \leq \exp\left(- \frac{(\epsilon^2 - \epsilon^3)n}{2(\frac{1}{\sigma^4} B+1)}\right).
\end{equation}
The left hand side can be rewritten as 
\begin{equation}
    \mathbb{P}\left(\| \frac{1}{\sqrt{n}} \frac{1}{\sigma} M u \|^2 \leq (1- \epsilon) \|u\|^2\right)
    =\mathbb{P}\left(\| \frac{1}{\sqrt{n}} M u \|^2 \leq \sigma^2 (1- \epsilon) \|u\|^2\right).
\end{equation}

2. Once again, let $R= \frac{1}{\sigma}M$. Clearly, for every integer $k$, we have 
\begin{equation}
\mathbb{E}\left(r_{ij}^{2k}\right)=\mathbb{E}\left(\frac{1}{\sigma^{2k}}m_{ij}^{2k}\right) \leq \frac{2k!}{2^k k!} L^{2k}.    
\end{equation}
Hence $R$ satisfies the conditions of Theorem~\ref{theorem 3.7} and we have 
\begin{equation}
\mathbb{P}\left( \| \frac{1}{\sqrt{n}} R u\|^2 \geq (1 + \epsilon) L^2 \|u\|^2\right) \leq \exp\left(-(\epsilon^2 -\epsilon^3)\frac{n}{4}\right).
\end{equation}
The left hand side of this equation is equal to
\begin{equation}
    \mathbb{P}\left(\| \frac{1}{\sqrt{n}} \frac{1}{\sigma} M u\|^2 \geq (1 + \epsilon) L^2 \|u\|^2\right)
    =\mathbb{P}\left(\| \frac{1}{\sqrt{n}} M u \|^2 \geq \sigma^2 (1 + \epsilon) L^2 \|u\|^2 \right).
\end{equation}

\end{proof}

Finally, we can use these results to prove an analogous theorem about the effect of multiplying by the random matrix $M$.

\begin{theorem}
Given the matrix $M\in\R^{n\times m}$, whose entries are each the difference of independent and identically distributed Bernoulli random variables with parameter $p\in(0,1)$,  it holds for any $\epsilon>0$ that
$$\mathbb{P}\left( (1-\epsilon) \|u-v\|^2 \leq \frac{1}{n\sigma^2}\|Mu- Mv\|^2 \leq (1+ \epsilon)  \|u-v\|^2\right) \geq 1-e^{-(\epsilon^2 - \epsilon^3) \frac{n}{4}} - e^{- \frac{(\epsilon^2 - \epsilon^3)n}{2(\frac{1}{\sigma^2} + 1)}},$$
for all $u,v\in\mathbb{R}^m$, where $\sigma^2=2p(1-p)$.
\end{theorem}

\begin{proof}
First, we need to show that our matrix $M$ satisfies the conditions of Corollary~\ref{corollary_3_8}. We've shown before that $\mathbb{E}(m_{ij}^2)=2p(1-p)$. When $p \in (0,1)$, it holds that $2p(1-p) > 0$. Since $m_{ij}^2 \in \{0,1\}$, we have that $(m_{ij}^2)^k = m_{ij}^2$ for any integer $k$. Therefore $\mathbb{E}(m_{ij}^{2k}) = \mathbb{E}(m_{ij}^2) = 2p(1-p)$ for all positive integers $k$. Picking $L= [2p(1-p)]^{-1/2}$ gives the relation
\begin{equation}
    \mathbb{E}(m_{ij}^{2k}) = 2p(1-p) \leq 1 \leq [2p(1-p)]^k \frac{(2k)!}{2^k k!} L^{2k},
\end{equation}
so that the condition in the second part of Corollary~\ref{corollary_3_8} is satisfied.

Finally, we can apply \Cref{corollary_3_8} to the vector $(u-v)$ to obtain the result, using the fact that
\begin{equation}
\begin{split}
    &\mathbb{P}\left( (1-\epsilon) \|u-v\|^2 \leq \frac{1}{n\sigma^2}\|Mu- Mv\|^2 \leq (1+ \epsilon) \|u-v\|^2\right) \\&= 1- 
    \mathbb{P}\left(\frac{1}{n}\|Mu- Mv\|^2 < (1-\epsilon)\sigma^2 \|u-v\|^2 \right) -
    \mathbb{P}\left( \frac{1}{n}\|Mu- Mv\|^2 > (1+ \epsilon) \sigma^2 \|u-v\|^2\right),
\end{split}
\end{equation}
and the choice $L^2=1/\sigma^2$.
\end{proof}

An easy way to guarantee that the multiplication of a vector with a matrix $y = Mx$ does not lose any information is to require that the matrix must be invertible, so that the initial vector $x$ can be retrieved exactly from $y$. Of course, in our case, $M$ has dimensions $n \times m$ with $n\gg m$, and thus can not be invertible. However, we can still ensure that $M$ has maximum rank $m$. In particular, we will show that with high probability any $m\times m$ submatrix of our random matrix $M$ will be invertible. This is based on modifying a result from \cite[Theorem~8.9]{tao2006random}.

\begin{theorem}[Invertibility] \label{theorem_inverse}
Given the matrix $M\in\R^{n\times m}$, where $n>m$, whose entries are each the difference of independent and identically distributed Bernoulli random variables with parameter $p\in(0,1)$, let $\mathcal{M}\in\R^{m\times m}$ be any square submatrix of $M$. Then, it holds for any $\epsilon>0$ that
$$\mathbb{P}\left(|\det(\mathcal{M})| \geq \left(2p(1-p)\right)^{m/2} \sqrt{m!} \exp({-m^{1/2+\epsilon}})\right)=1-o(1),$$
as $m\to\infty$. In particular, a submatrix $\mathcal{M}$ is invertible with probability at least $1-o(1)$ as $m\to\infty$.
\end{theorem}
\begin{proof}
The matrix $R=\frac{1}{\sigma}\mathcal{M}$ has entries that are bounded and have mean zero and variance one. This means it satisfies the hypotheses of \cite[Theorem~8.9]{tao2006random}, so we can conclude that 
\begin{equation}
    \mathbb{P}\left(|\det(R)| \geq \sqrt{m!} \exp({-m^{1/2+\epsilon}})\right)=1-o(1),
\end{equation}
as $m\to\infty$. Using the fact that $\sigma=\sqrt{2p(1-p)}$ and $\det(R)=\sigma^{-m}\det(\mathcal{M})$ gives the result.
\end{proof}

\Cref{theorem_inverse} says that the matrix $M\in\R^{n\times m}$ has submatrices that are likely to be invertible if the smaller dimension is sufficiently large. That is, they are likely to be invertible if the dimension of the feature vector that is the input to the transformation is sufficiently large. In practice, of course, any matrices will have finite size. In \Cref{fig:invertibility} we calculate the probability (averaged over $10^4$ independent realisations) of a given submatrix of $M$ being invertible, for two different values of $p$. For both $p=0.05$ and $p=0.1$, we see that a submatrix is invertible at least $99\%$ of the time when the dimension $m=100$ (in fact, when $p=0.1$, the $99\%$ threshold is reached when $m=48$). Since the dimension of the initial feature vector is likely to be large, any submatrices are highly likely to be invertible. In \Cref{sec:experiments}, we will generally use $p=0.05$ and $m=433$, in which case the submatrices are almost guaranteed to be invertible (the probability of being singular is negligibly small).

\begin{figure}
    \centering
    \includegraphics[width=0.6\linewidth]{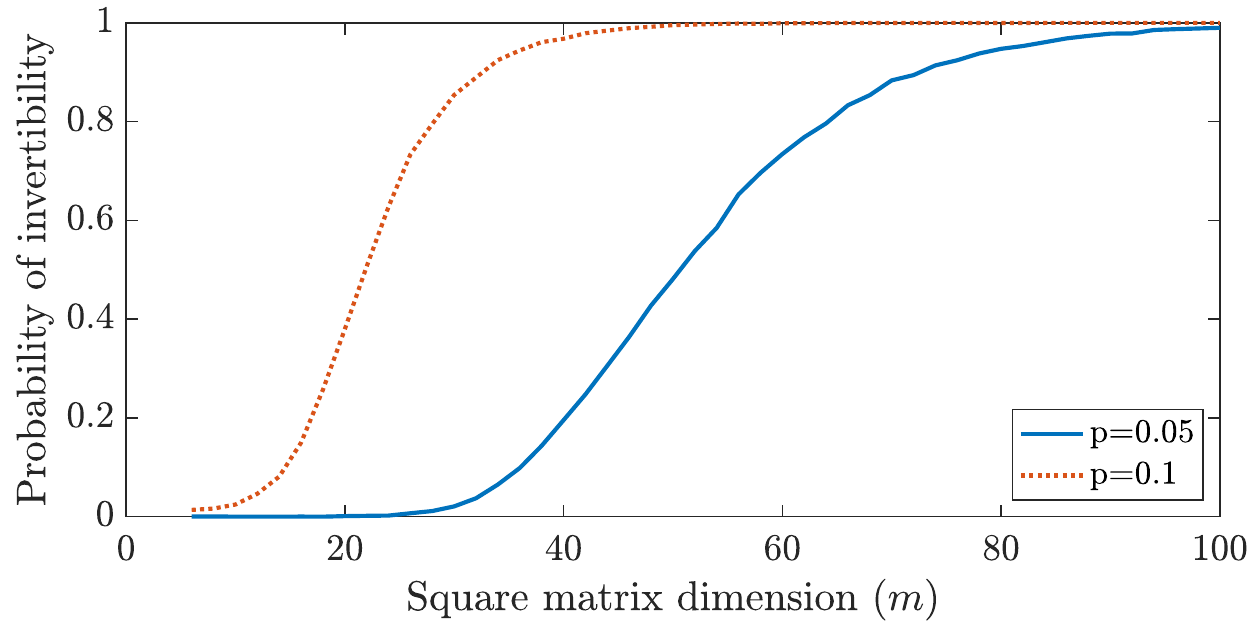}
    \caption{The probability of a submatrix of the random matrix $M$ being invertible. When the dimension of the square submatrix is arbitrarily large, the probability that it is invertible approaches one. However, for finite matrix dimensions, very high probabilities of invertibility can be obtained with relatively small matrices.}
    \label{fig:invertibility}
\end{figure}

\subsection{Properties of the cap operator}

The second part of the transformation $A$, defined in \eqref{eq:transformation}, is the application of a cap operator. This is a map $c_k:\R^n\to\R^n$ that retains the $k$ largest (in magnitude) entries of a vector and sets all the others to zero. A cap operator is a crude way to sparsify a vector, with a degree of sparsity that can be controlled by varying the parameter $k$. 

It is trivially the case that if $1\leq k\leq n$, then for any $x\in\mathbb{R}^n$ and any $p\in(0,\infty)$, it holds that
\begin{equation} \label{eq:capcts}
    \|x\|_\infty\leq\|c_k(x)\|_p\leq\|x\|_p.
\end{equation}
In fact, if $x$ is sufficiently sparse, in the sense that $\|x\|_p$ is small when $p$ is small, then $c_k(x)$ is close to $x$. Various results along these lines exist, we prove one such statement below. A different version can be found in \cite{donoho2006compressed}, for example.

\begin{theorem} \label{thm:cap}
Let $c_k : \mathbb{R}^n \longrightarrow \mathbb{R}^n$ be the cap operation that retains the $k$ largest (in magnitude) entries of a vector and sets all the others to zero. Then, for any $p \in (0,2)$, 
$$\| x - c_k(x) \|_2 \leq \|x\|_p (k+1)^{p+1}, \quad\text{for all } x\in\mathbb{R}^n.$$
\end{theorem}

\begin{proof}
We show this by induction on $k$. First, consider the case where $k=0$. By definition of the cap operation, $c_k(x)$ is then simply the zero vector, and we have $\|x\|_2 \leq \|x\|_p$ for $0 < p < 2$ by monotonicity of the norms. 

Suppose now that the property holds for $k-1$. Let $c_k^* (x)$ be the vector obtained from $x$ by keeping its $k$'th largest entry intact and setting all others to zero. We then have
$$ \|x - c_k(x) \|_2 = \|x - c_{k-1}(x) - c_k^* (x) \|_2 \leq \|x - c_{k-1} (x) \|_2 + \|c_k^*(x)\|_2.$$
From the inductive hypothesis, we have that $\|x - c_{k-1}(x) \|_2 \leq \|x\|_p k^{p+1}$. Moreover, $\|c_k^*(x)\|_2 \leq \|x\|_2 \leq \|x\|_p$ for $p \in (0,2)$. We obtain
$$\|x - c_k(x)\|_2 \leq \|x\|_p k^{p+1} + \|x\|_p \leq \|x\|_p (k+1)^{p+1},$$
where the final inequality holds from the fact that the function $f(x)=(x+1)^{p+1}-x^{p+1}-1$ satisfies $f'(x)=(p+1)[(x+1)^p-x^p]>0$ for $x>0$ and $f(0)=0$, meaning that $f(x)>0$ for all $x>0$. The result then follows by induction.
\end{proof}

The consequence of these results is that, thanks to \eqref{eq:capcts}, the effect of the cap operation is always bounded (in the sense that $\|c_k\|_{op}<1$) and, thanks to \Cref{thm:cap}, if the initial feature vector has some sparsity, then this effect will in fact be correspondingly small (in the sense that $c_k(x)$ is close to $x$). Conversely, in the numerical experiments presented in the following section, we will show that the algorithm performs well on classification problems even when the data (and also the projected data) are not sparse. This shows that the random projections succeed in encoding the important information in a small number of the coefficients (as will be demonstrated by the fact that if the random projection step is removed, then the classification accuracy drops).

\section{Numerical experiments} \label{sec:experiments}

We would like to explore the extent to which the bioinspired transformation \eqref{eq:transformation}, which yields sparse representations of signals through the use of random projections and a cap operation, can be used in classification problems. Recall that the main inspiration for this transformation was the function of the fruit fly's olfactory system, where the corresponding system's role is to facilitate the classification of odors. As a demonstrative classification problem, we chose to attempt musical genre classification. We used the GTZAN dataset \cite{gtzan} which consists of 30-second long extracts of music from 10 different genres. We used the scattering transform as a first step to obtain a suitable feature vector. The scattering transform is a cascading sequence of alternating wavelet transforms and modulus operators that outputs coefficients that are locally invariant to translations and stable to deformations \cite{deepspectrum}. The software for the scattering transform can be found online at \cite{scatnet}.

\subsection{Methods}

First, the scattering transform is applied to the dataset. The entire dataset contains 1000 music extracts in total, 100 from each genre. The scattering transform returns a $433 \times 20000$ array of coefficients, that is organized as an array of $433 \times 20$ scattering coefficients, for each of the 1000 samples. The first dimension is a feature dimension, while the second is a time dimension. To perform classification, we took the mean of the feature arrays in the time dimension, so that we ended up with feature vectors of length $433$. This allowed us to deal with feature vectors that do not possess a time dimension, and could thus be projected without fear of mixing up time-scale information. 

To obtain the random matrix used for projection, we first sampled two random matrices whose entries were independent Bernoulli random variables with parameter $p$ and then computed the difference of those two matrices. This gives the desired random matrix whose coefficients have mean zero, as described in \Cref{sec:maths}. After random projection, we applied the cap operation to the feature vectors, retaining the $k$ entries that were largest in absolute value and setting all others to zero. Given the resulting vector, classification was performed using a support vector machine. The linear support vector machine from Matlab's Classification Learner App was used.

We performed several experiments to understand the role of the parameters of the transformation $A$. In particular, $n$ is the dimension of the space into which our random matrix projects the signal vectors. The parameter $p$ corresponds to the Bernoulli parameter we use to sample the random matrix and, finally, the parameter $k$ indicates how many coefficients we keep intact after the cap operation. 

\subsection{Results}

We initially performed classification on the feature vectors without the bioinspired transformation, as a point of reference, to see if the random projection and cap operation improved or worsened the results. The resulting accuracy was consistently around $77 \%$. This is shown by the dotted lines in Figures~\ref{accuracy_parameter_p}, \ref{accuracy_parameter_n} and \ref{accuracy_parameter_k}. 

\begin{figure}
    \centering
    \includegraphics[width = 0.7\textwidth]{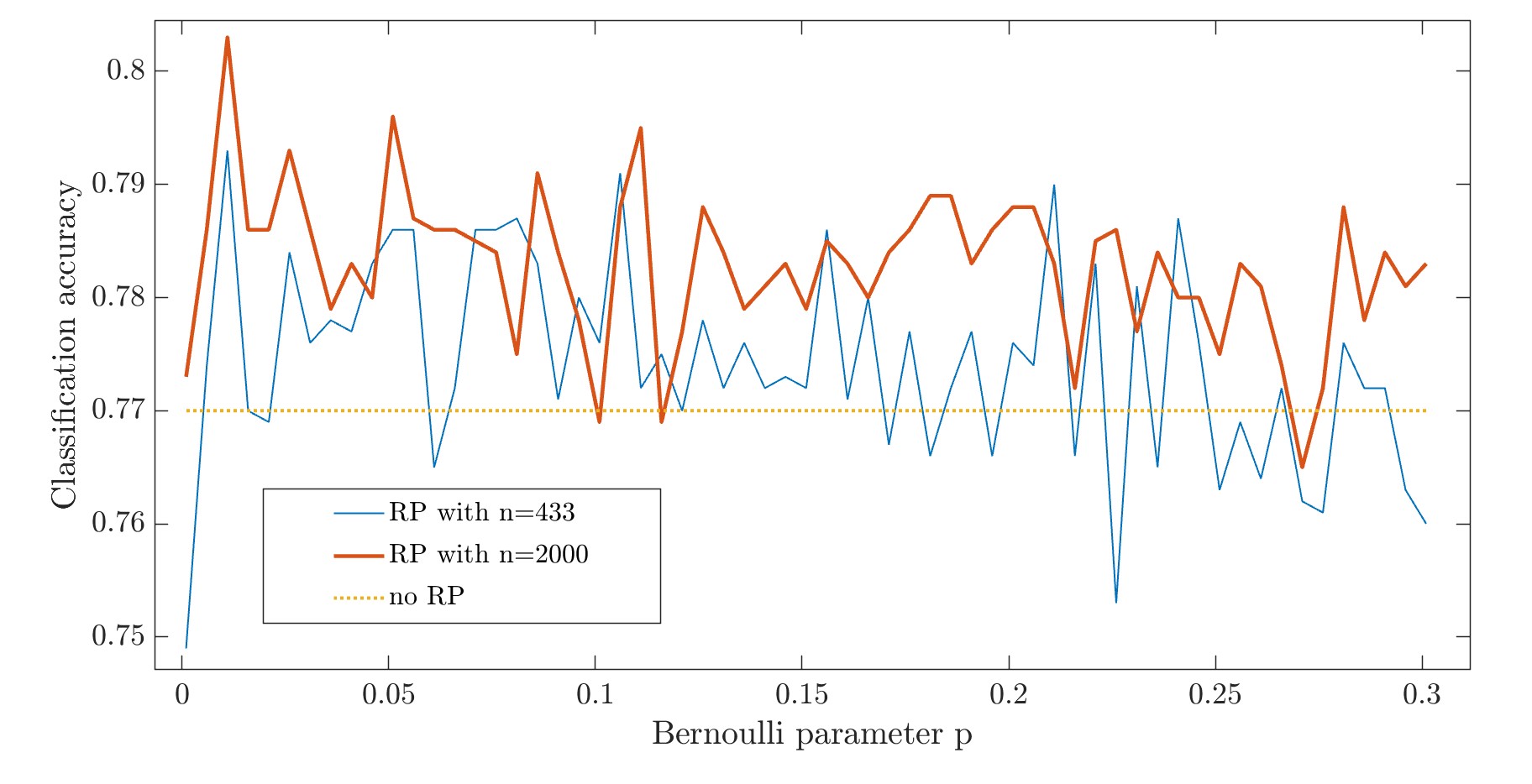}
    \caption{The classification accuracy is approximately preserved when random projections (RPs) are introduced and is not significantly affected by the Bernoulli parameter $p$. The classification accuracy is computed when using random projection of fixed dimension $n$, while varying the Bernoulli parameter and not using any cap operation.}
    \label{accuracy_parameter_p}
\end{figure}

Our next experiment was to add a random projection and understand the effect of varying the distribution parameter $p$. We performed those experiments both for both smaller ($n=433$) and larger ($n=2000$) image spaces, to see if there was a clear advantage to projecting the feature vectors into a much higher dimensional space. The results can bee seen in Figure~\ref{accuracy_parameter_p}, which suggest that there is no clear relation between the Bernoulli parameter $p$ and classification accuracy. As shown in \Cref{sec:maths}, the entries of our random matrix $M$ are zero with probability $\mathbb{P}(m_{ij}=0) = 2p^2-2p+1$. This expression is strictly decreasing for $p<0.5$, so the entries of $M$ are less likely to be zero when $p$ is larger. It is therefore in our interest to keep $p$ small in subsequent experiments, so that our random matrix is likely to be sparser, which speeds up the calculations. In the following experiments, given that the effect on classification appears to be minimal, we fix the value $p=0.05$. 

As well as suggesting that the value of $p$ does not greatly influence the classification accuracy, Figure~\ref{accuracy_parameter_p} may suggest that applying a random projection slightly improves the classification accuracy. This is particularly the case for the larger dimension $n=2000$, where the average success rate was more than a percentage point higher than for the classification without any random projection. However, this difference is not significant enough to be able to draw convincing conclusions at this point. 

\begin{figure}
    \centering
    \includegraphics[width = 0.6\textwidth]{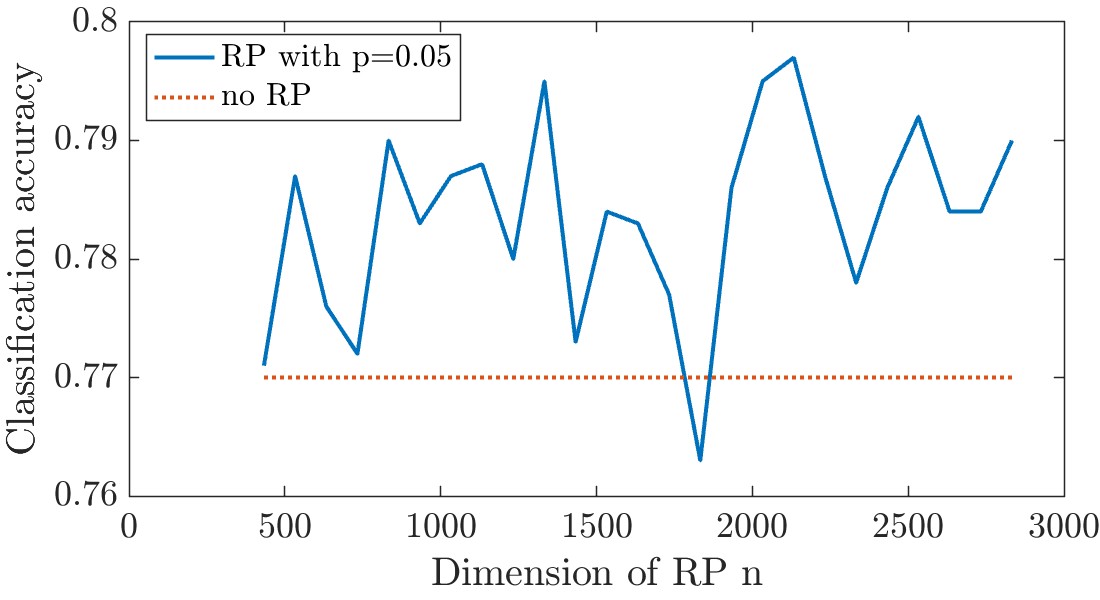}
    \caption{The classification accuracy is approximately preserved when random projections (RPs) are introduced and is not significantly affected by the dimension of the space into which we project. The classification accuracy is computed when using random projection of varying dimension $n$, fixed Bernoulli parameter $p=0.05$ and no cap operation.}
    \label{accuracy_parameter_n}
\end{figure}

To evaluate the effect of the projection dimension $n$, we fixed the Bernoulli parameter at $p=0.05$ and increased the dimension of the random matrix from $n=433$ (the dimension of the feature vectors output by the scattering transform) up to $n=2833$, adding $100$ rows every time. The results can be seen in Figure~\ref{accuracy_parameter_n}. The accuracy is relatively stable across all values of $n$, and is similar to the performance of the original feature vectors without random projection. This result is maybe not surprising, as we're not adding information with the random projection, but instead merely randomly shuffling it. It is however reassuring to observe that information isn't lost or corrupted by adding a random projection. Once again, there appears to even be a slight improvement in classification accuracy, thanks to the introduction of the random projection.

\begin{figure}
    \centering
    \includegraphics[width = 0.6\textwidth]{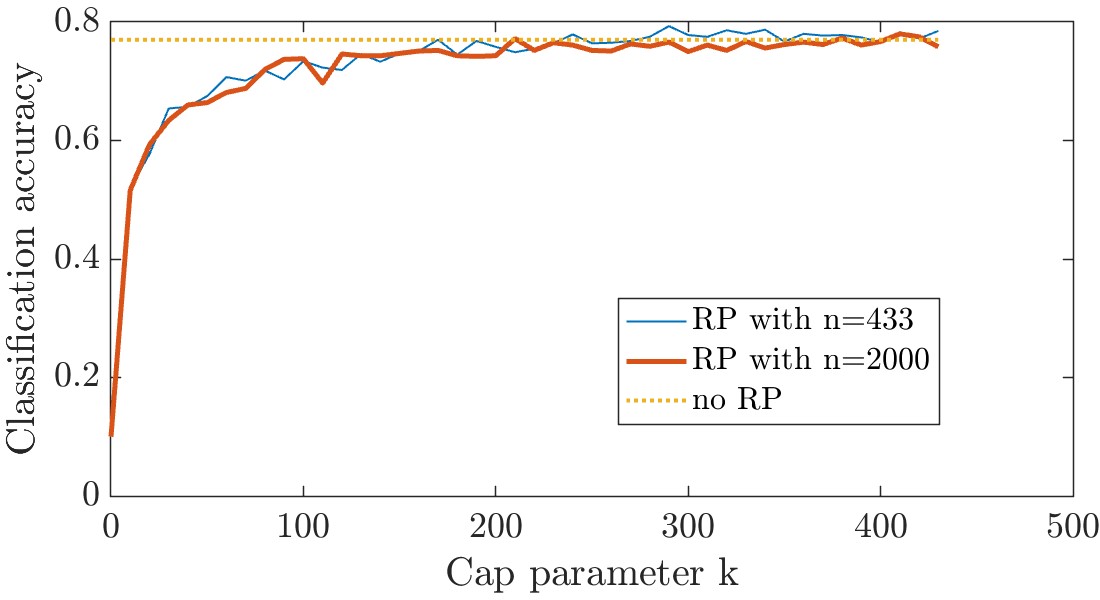}
    \caption{For sufficiently large cap parameter $k$, the introduction of a cap operation does not greatly influence the classification accuracy. This holds even when $k$ is significantly smaller than the dimension of the initial feature vector (433). Classification accuracy is computed when using random projections (RPs) of fixed dimensions $433$ and $2000$, Bernoulli parameter $p=0.05$ and a varying cap parameter $k$.}
    \label{accuracy_parameter_k}
\end{figure}

\begin{figure}
    \centering
    \includegraphics[width = 0.6\textwidth]{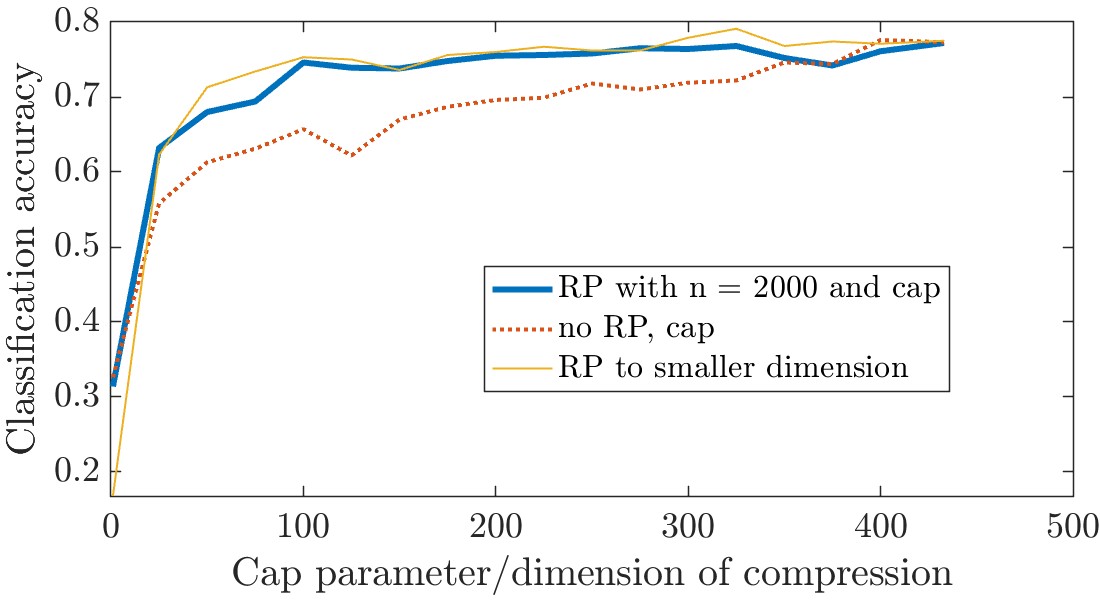}
    \caption{The application of a random projection (RP) together with a cap consistently shows higher accuracy than the use of a cap alone. On the other hand, a simple random projection to a space of smaller dimension performs similarly to the random projection and cap. All three methods show relatively good classification accuracy, as long as the cap/dimension of compression is above $100$. Both random projections in this figure were performed with a Bernoulli parameter $p = 0.05$. }
    \label{comparison_rp_id}
\end{figure}

The effect of the cap operation on the classification accuracy can also be tested. We kept the Bernoulli parameter at $p=0.05$ and added a cap operation with varying parameter $k$. We recorded classification accuracy for two different random projections; one with dimension $n=433$ and the other with $n=2000$. The results can be found in Figure~\ref{accuracy_parameter_k}. For very small values of $k$, the classification accuracy drops away quickly (down to a limiting value of $10\%$ when $k=0$, in which case the ten classes are allocated without any retained information). However, the accuracy is rather stable for larger values of $k$. In particular, $k \approx 200$ is sufficient in both cases to attain a classification accuracy above $75 \%$. This result is noteworthy as $k=200$ is less than half the size of the initial feature vectors outputted by the scattering transform. As previously with Bernoulli parameter $p$, it is in our interest to keep $k$ small, so that the final feature vector is as sparse as possible. However, it is worth noting that our experiments revealed that increasing the cap parameter $k$ did not dramatically decrease the training time of the support vector machine: with the Bernoulli parameter set at $p=0.05$ and no cap operation the training time was $24.0$ seconds, while setting a cap of $k=200$ led to a training time of $23.3$ seconds. This is because the linear support vector machine is not set up to be able to take advantage of the sparsity.

Taken together, these experiments show that projecting randomly into a space of seemingly arbitrary dimension with a matrix with small Bernoulli parameter before capping to leave just a couple of hundred entries gives an algorithm with good classification accuracy. In particular, we take projection into a space with dimension $n=2000$, with Bernoulli parameter $p=0.05$ and cap parameter $k=200$ as our gold standard. Figure~\ref{comparison_rp_id} shows the comparison of this transformation with the application of a cap alone. The full transformation (with random projection followed by the cap) consistently performs better than the cap alone, showing that the application of a random projection is important for retaining information when the vectors are truncated using the cap operator. We also compared the performance of our transformation $A$ with a simple random compression. This random compression was performed by multiplying by the same random matrix used in our bioinspired transformation, with a Bernoulli parameter $p = 0.05$, but with the dimension decreased. The image dimension of the random compression is shown on the horizontal axis of Figure~\ref{comparison_rp_id}, so that it can be compared to the cap parameter in the two other algorithms. All three sets of results show that the accuracy does not drop significantly as long as the number of coefficients retained is above $100$, thus suggesting that we can easily decrease the dimension of the feature vectors (which initially have $433$ entries), thus making them sparser, without impeding our ability to classify them. 

\begin{figure}
    \centering
    \includegraphics[width = 0.6\textwidth]{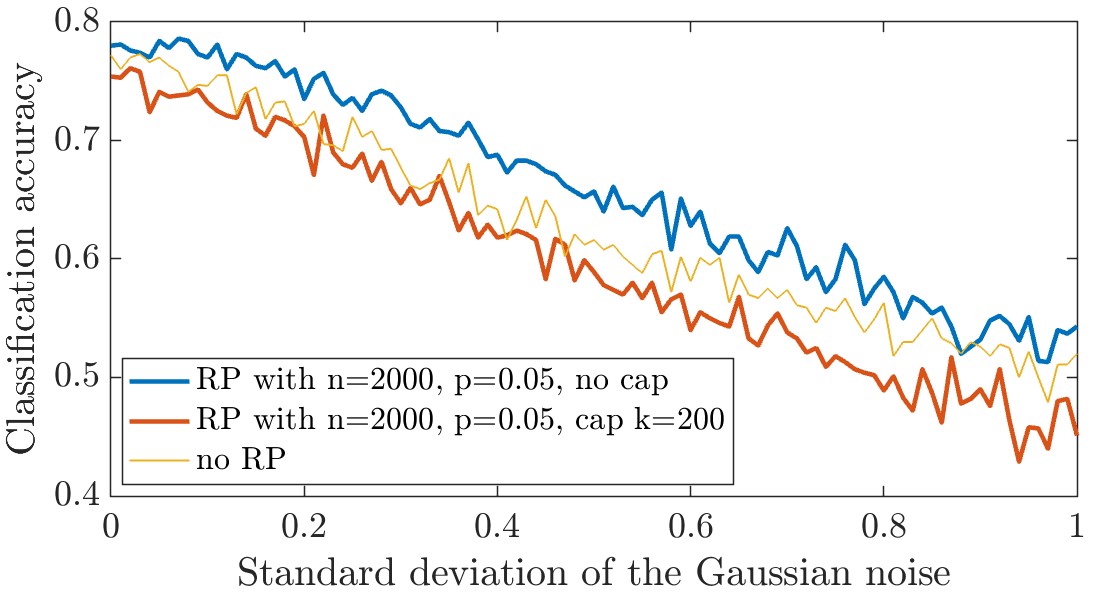}
    \caption{The classification accuracy shows improved robustness to noise when random projections (RPs) are introduced. This effect is, however, lost when a cap operator is added. Classification accuracy is computed when adding Gaussian noise of mean zero and varying standard deviation to the initial feature vectors.}
    \label{accuracy_noise}
\end{figure}

Our final experiment focused on determining whether the addition of the bioinspired transformation $A$ could improve the robustness of the classification. We added random Gaussian noise to the feature vectors outputted by the scattering transform and compared three cases: classifying the noisy vectors directly without any transformation, classifying the noisy vectors that have been randomly projected to a space of dimension $n = 2000$ and, finally, classifying the noisy vectors after applying both the random projection and a cap operation with a parameter $k = 200$. The results for these three cases are reported in Figure~\ref{accuracy_noise}. In each case the noise was generated independently with zero mean and increasing standard deviation. It is clear from Figure~\ref{accuracy_noise} that adding the random projection increases the robustness of the classification. Classification with the projected vectors consistently performs a few percent better than without any random projection. Conversely, the addition of a cap operation seems to yield slightly worse results, suggesting that there is a trade off between robustness and sparsity.

\section{Conclusions}

Our results demonstrate yet another successful application of biomimicry, this time to a classification problem. We designed a two-step signal transform that used random projections and a cap operation, inspired by the function of the fruit fly's olfactory system. This transformation, when paired with a support vector machine, gives you the ability to sparsify the data while preserving the classification accuracy. Our experiments showed that it also leads to robustness benefits, giving improved classification accuracy when random errors are added to the data. Perhaps most importantly, the signal transform is very simple and requires very little computational power to execute, giving a distinct advantage over more intricate approaches.

Sparsity and robustness are important properties for biological sensing systems. Animals have limited neural bandwidth so need to be able to encode information as efficiently as possible. Sparsifying data is one way of achieving this. Other strategies include the using compressive nonlinearities to rescale data \cite{hudspeth2010critique} and using known statistical properties of target data sets to obtain low-dimensional representations \cite{ammari2020biomimetic, attias1997coding, gervain2019efficient}. Robustness is similarly important for an animal's ability to understand its noisy environment and there are many examples of biological systems demonstrating remarkable robustness properties, see \emph{e.g.} \cite{davies2021robustness, gervain2019efficient, mesgarani2014mechanisms}. These outstanding properties should motivate further studies of biological and biologically inspired systems, with the aim of understanding the fundamental mechanisms, so that they can be implemented into novel solutions to challenging problems.

\section*{Acknowledgements}

The work of BD was partly supported by the EC-H2020 FETOpen project BOHEME under grant agreement No. 863179.

\section*{Data availability}

The code used for the numerical experiments in this work is available at \url{https://doi.org/10.5281/zenodo.6642660}.

\bibliographystyle{abbrv}
\bibliography{refs}
\end{document}